\documentclass[conference,letterpaper]{IEEEtran2026}
\usepackage[left=0.74in,right=0.64in,top=0.7in,bottom=0.95in]{geometry}

\IEEEoverridecommandlockouts
\newcommand{\Latexfilepath}{.}
\usepackage{\Latexfilepath/Bermudez}

\begin{document}

\title{Machine Unlearning \\ for Gibbs Supervised Learning Algorithms}

 \author{
   \IEEEauthorblockN{
   Yaiza Bermudez\IEEEauthorrefmark{1}, 
   Samir M. Perlaza\IEEEauthorrefmark{1}\IEEEauthorrefmark{2}\IEEEauthorrefmark{4}, and
   I\~{n}aki Esnaola\IEEEauthorrefmark{3}\IEEEauthorrefmark{4}\\
   Emails: name.lastname@inria.fr and esnaola@sheffield.ac.uk
                     }
   \IEEEauthorblockA{\IEEEauthorrefmark{1}%
                     Centre Inria d'Université Côte d'Azur, INRIA, 
                     Sophia Antipolis, France.}
   \IEEEauthorblockA{\IEEEauthorrefmark{2}%
                     Laboratoire GAATI, Universit\'{e} de la Polyn\'{e}sie fran\c{c}aise, Fa`a`\={a}, French Polynesia.}
   \IEEEauthorblockA{\IEEEauthorrefmark{3}%
                     School of Electrical and Electronic Engineering, University of Sheffield, Sheffield,
United Kingdom.}   
   \IEEEauthorblockA{\IEEEauthorrefmark{4}%
		     ECE Dept. Princeton University, Princeton, 08544 NJ, USA.
\thanks{This work is supported in part by the European Commission through the H2020-MSCA-RISE-2019 project 872172; the French National Agency for Research (ANR)  through the Project ANR-21-CE25-0013 and the project ANR-22-PEFT-0010 of the France 2030 program PEPR Réseaux du Futur; and in part by the Agence de l'innovation de défense (AID) through the project UK-FR 2024352. }                     } 
  }
 \maketitle 

\begin{abstract} 
In this paper, a method for achieving exact unlearning for Gibbs supervised learning algorithms is proposed using a variational formulation inspired by empirical risk minimization subject to relative entropy regularization (ERM-RER). Such a method consists of maximizing the expected empirical risk over the dataset to be unlearned subject to a regularization by relative entropy with respect to the original algorithm. The optimization variable is a probability measure on the models; and the solution is another Gibbs probability measure that represents a new Gibbs supervised learning algorithm.
The method guarantees exact unlearning in the sense that the new Gibbs algorithm coincides  in distribution with the algorithm  that would have been obtained by retraining from scratch on the  dataset to be retained. 
As a byproduct, a framework for reweighting data points in ERM-RER by strategically choosing both the reference measure and the regularization factor is obtained. In this framework, exact unlearning is the special case in which zero-weight is assigned to the contribution of the data points to be unlearned. More generally, depending on the  choice of certain parameters, data points can be up-weighted or down-weighted in ERM-RER problems for particular purposes, e.g., controlling the generalization error of Gibbs algorithms. This paves the way for new constructive or adversarial views on classical reweighting data points in ERM-RER.  
\end{abstract}
\enlargethispage{-0.10in}
\vspace{-1.2ex}
\section{Introduction}

Modern learning systems are routinely trained on data that may later become unavailable or impermissible to use, e.g., due to user requests, contractual constraints, or regulatory requirements such as the right to erasure in Article~17 of the GDPR~\cite{gdpr2016}, also known as the ``Right to be Forgotten".  In this context, \emph{machine unlearning} refers to the task of removing the influence of a part of the training dataset from a learning algorithm while avoiding the computational cost of retraining from scratch \cite{Nguyen2025Unlearning,Sai2024Machine,Xu2024Solutions}. Early formulations of data deletion in learning formalize this task as producing an updated algorithm that is statistically consistent with training on the retained data only~\cite{ginart2019making,cao2015towards}. 
A particularly stringent requirement is \emph{exact unlearning}, in which the distribution of the post-unlearning algorithm is required to be identical to the algorithm obtained by retraining on the dataset to be retained~\cite{cao2015towards}. \vspace{-1ex}These guarantees are difficult to obtain and practical approaches often trade computational efficiency for approximate removal~\cite{bourtoule2021machine,Egger2025ModelSplitting,Jiang2024FedUHB, jiang2025certifying}, often referred to as \emph{approximate unlearning}~\cite{guo2020certified}.

\vspace{-0.7ex}
In this work, exact unlearning is formulated as a relative-entropy-regularized variational problem, which allows using existing tools previously developed in the framework of empirical risk minimization with $f$-divergence regularization \cite{perlaza2024empirical, zou2024Generalization, daunas2025asymmetry, perlaza2023validation, daunas2024equivalence}. 
The main contribution of this work is an exact unlearning method for Gibbs algorithms. More specifically, the method consists in implementing a Gibbs algorithm via the maximization of the empirical risk with respect to the dataset to be unlearned using as reference measure the Gibbs algorithm from which data must be unlearned.  
Interestingly, this method guarantees exact unlearning in the sense that the post-unlearning algorithm coincides with the solution to an ERM-RER problem with respect to the retained training dataset.
The key observation is that  exact unlearning is achieved thanks to the strategic choice of both the reference measure, as in \cite{bermudez2026decentralized}; and the regularization factor, as in \cite{Bu2024Towards,medina2022robustness,ray2023asymptotics}. This choice results in an effect that can be assimilated to a reweighting of the contribution of the data points to be unlearned. In particular, unlearning consists in assigning zero-weight to those data points. From this perspective, the results presented here can be placed in a more general scenario than unlearning.  
In particular, it can be placed in the scenario of reweighting data points in ERM-RER.  %
More generally, this reweighting might be oriented to decrease the generalization error~\cite{perlaza2024generalization} of Gibbs algorithms  by controlling the contribution of particular data points. On the other hand, reweighting data points in ERM-RER might be performed adversarially. In this case, the objective might be to increase the generalization error of Gibbs algorithms by up-weighting the outliers within the training dataset. 

\enlargethispage{-0.11in}
The paper is organized as follows. Section~\ref{SecSupervisedMachineLearning} introduces the notation and formalizes the supervised learning setting. Section~\ref{SecGibbsAlgorithms} defines Gibbs conditional probability measures and their characterizations via ERM-RER optimization problems. Section~\ref{SecMainResult} introduces the main result, which consists in an exact unlearning method for Gibbs algorithms. Finally, Section~\ref{SecConclusion} concludes the paper, and discusses open research paths.

\section{Supervised Machine Learning}
\label{SecSupervisedMachineLearning}

Let~$\set{M}$,~$\set{X}$ and~$\set{Y}$, with~$\set{M} \subseteq \reals^{d}$ and~$d \in \ints$, be sets of \emph{models}, \emph{patterns}, and \emph{labels}, respectively.  
The training data available is assumed to be partitioned into~$K$ smaller datasets. For all~$k \inCountK{K}$, the~$k$-th dataset consists of~$n_k$ data points~$(x_{k,1}, y_{k,1}),\ldots,(x_{k,n_k}, y_{k,n_k})$, which are elements of the set~$\set{Z} \triangleq \set{X} \times \set{Y}$. Such data points form the dataset~$\vect{z}_k \in \set{Z}^{n_k}$, which can be explicitly written as  
\begin{IEEEeqnarray}{rCl}
\label{EqDatasetK}
\vect{z}_k & \triangleq & \autoparent{(x_{k,1}, y_{k,1}), \ldots, (x_{k,n_k}, y_{k,n_k})}\in \set{Z}^{n_k}.
\vspace{-1ex}
\end{IEEEeqnarray}
The aggregated training dataset is then defined explicitly as,
\begin{IEEEeqnarray}{rCl}
\label{EqDatasetZero}
\vect{z}_0 & \triangleq & \autoparent{\vect{z}_1, \ldots, \vect{z}_K } \in \set{Z}^{n_1} \times \set{Z}^{n_2} \times \dots \times \set{Z}^{n_K},
\end{IEEEeqnarray}
with size~$n_0 \triangleq \sum_{k=1}^{K} n_k$ data points.
Given a model~$\vect{\theta} \in \set{M}$, the loss induced by such a model with respect to a data point~$\autoparent{x,y} \in \set{Z}$ is~$\ell(x, y, \vect{\theta})$, where the function 
\begin{IEEEeqnarray}{rCl}
\label{EqLossFunction}
\ell: \set{Z} \times \set{M}\to [0, +\infty),
\vspace{-1ex}
\end{IEEEeqnarray}
is referred to as the \emph{loss function} and  is assumed to be Borel measurable.
Using such a loss function, the \emph{empirical risk} induced by a model~$\vect{\theta} \in \set{M}$, with respect to a dataset~$\vect{z}_k \in \set{Z}^{n_k}$, with~$k \in \{0,1,\ldots,K\}$,  is determined by the function
\begin{IEEEeqnarray}{rCl}
\label{EqEmpiricalRiskk}
\mathsf{L}_k:  \function{\set{Z}^{n_k} \times \set{M}}{[0, +\infty)}{\autoparent{\vect{z}_k, \vect{\theta}}}{\frac{1}{n_k}\sum_{i=1}^{n_k}  \ell\left(x_{k,i}, y_{k,i}, \vect{\theta} \right).}
\end{IEEEeqnarray}
The functions~$\mathsf{L}_0, \mathsf{L}_1, \ldots, \mathsf{L}_K$ exhibit the following property.
\begin{lemma}
\label{LemmaJan9at11h49in2026Sophia}
Consider the training dataset~$\vect{z}_0$ in~\eqref{EqDatasetZero}, then the empirical risk functions~$\mathsf{L}_0, \mathsf{L}_1, \ldots, \mathsf{L}_K$ in~\eqref{EqEmpiricalRiskk} satisfy for all~$\vect{\theta} \in \set{M}$,
\vspace{-3ex}
\begin{IEEEeqnarray}{rCl}
\label{EqJan6at09h42in2026inSophia}
\mathsf{L}_0\autoparent{\vect{z}_0, \vect{\theta}} =  \sum_{k=1}^{K} \frac{n_k}{n_0} \mathsf{L}_k \autoparent{\vect{z}_k, \vect{\theta}},
\vspace{-1ex}
\end{IEEEeqnarray}
\end{lemma}
\begin{IEEEproof}
The proof is similar to the proof of \cite[Lemma~4]{perlaza2023validation}.
\end{IEEEproof}
The set of all probability measures on the measurable space $\left(\set{M},\mathscr{F}_{\set{M}}\right)$ is denoted by $\triangle\!\left(\set{M}\right)$.  
Using this notation, the relative entropy or KL-divergence is defined hereunder.
\begin{definition}[Relative Entropy]
Given a probability measure $P$ and a $\sigma$-finite measure $Q$, both on the same measurable space, with $P$ absolutely continuous with respect to $Q$. The relative entropy of $P$ with respect to $Q$ is 
\begin{eqnarray}
\KL{P}{Q} &\triangleq& \int \RND{P}{Q}(\vect{\theta}) \log \left(\RND{P}{Q}(\vect{\theta})\right)  \mathrm{d}Q(\vect{\theta}).
\end{eqnarray}
\end{definition}
\enlargethispage{-\baselineskip}
The set of all probability measures in $\set{M}$ conditioned on an element of ${\set{Z}^{n_k}}$ is denoted by $\simplex{\set{M} \mid \set{Z}^{n_k}}$. 
Moreover, the set of probability measures in $\simplex{\set{M}}$ that are absolutely continuous with respect to $Q_k$ is denoted by $\simplexabs{Q_k}{\set{M}}$.
Using this notation,  a supervised machine learning algorithm is represented by a conditional probability measure, as defined hereunder. %
\begin{definition}[Algorithm]
A conditional probability measure~$P_{\vect{\Theta} \mid \vect{Z}_k}\in \simplex{\set{M} \mid \autoparent{\set{X} \times \set{Y}}^{n_k}}$ is said to represent a supervised machine learning algorithm. The
instance of such an algorithm trained upon the dataset~$\vect{z}_k$ in~\eqref{EqDatasetK} is denoted by~$P_{\vect{\Theta} \mid \vect{Z}_k = \vect{z}_k} \in \simplex{\set{M}}$.
\end{definition}
A class of algorithms that are central in this work are known as Gibbs algorithms. These algorithms are introduced in the following section.

\section{Gibbs Algorithms}
\label{SecGibbsAlgorithms}
Gibbs algorithms are represented by Gibbs conditional probability measures.
Such conditional probability measures are parametrized by the empirical risk function~$\mathsf{L}_k$ in~\eqref{EqEmpiricalRiskk}; a~$\sigma$-finite measure~$Q \in \simplex{\set{M}}$; and a dataset~$\vect{z}_k \in \set{Z}^{n_k}$, with~$k \in \{0, 1, \ldots, K\}$. 
Using this notation, the definition of Gibbs conditional probability measures on the models is presented hereunder.

\begin{definition} 
\label{DefGibbsModels}
Given the function~$\mathsf{L}_k$ in~\eqref{EqEmpiricalRiskk}, with~$k \in \{0, 1, \ldots, K\}$; a~$
\sigma$-finite measure~$Q \in \simplex{\set{M}}$; and 
a~$\lambda_k \in \reals\setminus\lbrace 0 \rbrace$, the probability measure~$P^{(Q, \lambda_k)}_{\vect{\Theta} \mid \vect{Z}_k}\in \triangle\left(\set{M} \mid \set{Z}^{n_k} \right)$ is said to be an~$(\mathsf{L}_k, Q, \lambda_k)$-Gibbs conditional probability measure if
 \begin{IEEEeqnarray}{rCl}
 \label{EqJan16at18h37in2026Sophia}
 \mbox{$\forall \vect{z}_k \in \set{Z}^{n_k}$, }  \int \exp\left( -\frac{1}{\lambda_k} \, \mathsf{L}_k\autoparent{\vect{z}_k , \vect{\theta}}\right)\mathrm{d}Q\left( \vect{\theta}\right) < + \infty; 
 \end{IEEEeqnarray}
and for all~$\left(\vect{z}_k, \vect{\theta} \right) \in \set{Z}^{n_k} \times \supp Q$,
\begin{IEEEeqnarray}{rcl}
\label{EqGibbsRNDFixedDataset}
\frac{\mathrm{d}P_{\vect{\Theta}\mid\vect{Z}_k = \vect{z}_k}^{(Q, \lambda_k)}}{\mathrm{d}Q } \autoparent{\vect{\theta}} & = & \frac{ \exp\autoparent{- \frac{1}{\lambda_k} \EmpRisk{k}{\vect{z}_k}{\vect{\theta}}}}{\int \exp\left( -\frac{1}{\lambda_k} \, \mathsf{L}_k\autoparent{\vect{z}_k , \vect{\nu}}\right)\mathrm{d}Q\left( \vect{\nu}\right)}.
\end{IEEEeqnarray}
\end{definition}
Note that if $\lambda_k<0$, the condition in~\eqref{EqJan16at18h37in2026Sophia} may fail. Therefore, throughout this work, the parameter $\lambda_k$ is chosen so that~\eqref{EqJan16at18h37in2026Sophia} holds.
Note that, while~$P_{\vect{\Theta}\mid \vect{Z}_k}^{(Q, \lambda_k)}$ in~\eqref{EqGibbsRNDFixedDataset} is referred to as a Gibbs conditional probability measure, the measure~$P_{\vect{\Theta}\mid\vect{Z}_k = \vect{z}_k}^{(Q, \lambda_k)}$, obtained by conditioning upon a given dataset~$\vect{z}_k\in\set{Z}^{n_k}$, is referred to as a Gibbs probability measure.
Consider the functional~$\mathsf{R}_{\vect{z}_k}$ defined as follows, 
\begin{IEEEeqnarray}{rCl}
\label{EqExpEmpRiskModels}
\mathsf{R}_{\vect{z}_k} : \function{ \simplex{\set{M}}}{[0, +\infty)}{P}{\int \mathsf{L}_k\autoparent{\vect{z}_k,\vect{\theta}} \mathrm{d}P\autoparent{\vect{\theta}},}
\end{IEEEeqnarray}
where the function~$\mathsf{L}_k$ is defined in~\eqref{EqEmpiricalRiskk}.
The Gibbs probability measure~$P_{\vect{\Theta}\mid\vect{Z}_k = \vect{z}_k}^{(Q, \lambda_k)}$ in~\eqref{EqGibbsRNDFixedDataset} is related to the following optimization problem, under the assumption that~$\lambda_k > 0$:
\begin{IEEEeqnarray}{rCl}
\label{EqJul25at9h55in2025NiceA}
\min_{P \in \triangle_{Q} \autoparent{\set{M}}} & & \mathsf{R}_{\vect{z}_k} \left( P \right) + \lambda_k \KL{P}{Q}.
\end{IEEEeqnarray}
Alternatively, when~$\lambda_k < 0$, such a Gibbs probability measure is related to the following optimization problem:
\begin{IEEEeqnarray}{rCl}
\label{EqJul25at9h56in2025NiceB}
\max_{P \in \triangle_{Q} \autoparent{\set{M}}} & & \mathsf{R}_{\vect{z}_k} \left( P \right) + \lambda_k \KL{P}{Q},
\end{IEEEeqnarray}
where the functional~$\mathsf{R}_{\vect{z}_k}$ is defined in~\eqref{EqExpEmpRiskModels}.
The importance of these optimization problems stems from the fact that the unlearning problem can be posed as the optimization problem in~\eqref{EqJul25at9h56in2025NiceB}.
The following lemma shows the connections between the Gibbs probability measure~$P_{\vect{\Theta}\mid \vect{Z}_k = \vect{z}_k}^{(Q, \lambda_k)}$ in~\eqref{EqGibbsRNDFixedDataset} and the optimization problems in~\eqref{EqJul25at9h55in2025NiceA} and~\eqref{EqJul25at9h56in2025NiceB}.
\enlargethispage{-0.18in}
\begin{lemma}
\label{LemmaJan8at15h57in2026inSophia}
The~$(\mathsf{L}_k, Q,\lambda_k)$-Gibbs probability measure~$P_{\vect{\Theta}\mid \vect{Z}_k = \vect{z}_k}^{(Q, \lambda_k)}$ in~\eqref{EqGibbsRNDFixedDataset} is the unique solution to~\eqref{EqJul25at9h55in2025NiceA} \vspace{1ex}(respectively, to~\eqref{EqJul25at9h56in2025NiceB}), if~$\lambda_k > 0$ (respectively, if~$\lambda_k < 0$).
\end{lemma}
\begin{IEEEproof}
The proof is immediate from \cite[Lemma~1]{PerlazaEntropy2025}.
\end{IEEEproof}
Two optimization problems that are also closely related to the probability measure~$P_{\vect{\Theta}\mid\vect{Z}_k = \vect{z}_k}^{(Q, \lambda_k)}$ in~\eqref{EqGibbsRNDFixedDataset} are the following. When~$\lambda_k > 0$, the optimization problem is:
\begin{subequations}
\vspace{-2ex}
\label{EqJanuary1at13h54in2026atNice}
\begin{IEEEeqnarray}{rCl}
\min_{P \in \triangle_{Q} \autoparent{\set{M}}} & & \mathsf{R}_{\vect{z}_k} \left( P \right) \\
\text{s.t.} & \quad &  \KL{P}{Q} \leqslant \gamma_k,
\end{IEEEeqnarray} 
\end{subequations}
for some~$\gamma_k>0$.
Alternatively, when~$\lambda_k < 0$, the optimization problem is:
\begin{subequations}
\vspace{-3ex}
\label{EqOct29at18h32in2025inAntibesB}
\begin{IEEEeqnarray}{rCl}
\max_{P \in \triangle_{Q} \autoparent{\set{M}}} & & \mathsf{R}_{\vect{z}_k} \left( P \right) \\
 && \mbox{s.t.}~\KL{P}{Q} \leq \gamma_k.
\end{IEEEeqnarray}
\end{subequations}
The following lemma formalizes these observations.
\begin{lemma}\label{LemmaJanuary1at14h24in2026}
Consider the~$(\mathsf{L}_k, Q,\lambda_k)$-Gibbs probability measure~$P_{\vect{\Theta}\mid \vect{Z}_k = \vect{z}_k}^{(Q, \lambda_k)}$ in~\eqref{EqGibbsRNDFixedDataset}  and assume that~$\lambda_k > 0$ (respectively,~$\lambda_k < 0$) is such that
\begin{IEEEeqnarray}{rCl}
 \KL{P_{\vect{\Theta}\mid\vect{Z}_k = \vect{z}_k}^{(Q, \lambda_k)}}{Q} & = & \gamma_k,
\end{IEEEeqnarray}
with $\gamma_k$ in \eqref{EqJanuary1at13h54in2026atNice} (respectively, in \eqref{EqOct29at18h32in2025inAntibesB}).
Then, the measure~$P_{\vect{\Theta}\mid \vect{Z}_k = \vect{z}_k}^{(Q, \lambda_k)}$  is the unique solution to~\eqref{EqJanuary1at13h54in2026atNice} (respectively, to~\eqref{EqOct29at18h32in2025inAntibesB}).
\end{lemma}
\begin{IEEEproof}
The proof follows from \cite[Lemma~4]{PerlazaEntropy2025}. 
\end{IEEEproof}
Lemma~\ref{LemmaJanuary1at14h24in2026} implies that the probability measure~$P_{\vect{\Theta}\mid\vect{Z}_k = \vect{z}_k}^{(Q, \lambda_k)}$ in~\eqref{EqGibbsRNDFixedDataset} is the one that minimizes (respectively, maximizes) the training empirical risk over all probability measures in the following neighborhood of~$Q$,
\begin{IEEEeqnarray}{rCl}
\label{EqJanuary1at19h39in2026HomeNice}
\left\lbrace P \in \simplexabs{Q}{\set{M}} :  \KL{P}{Q} \leqslant  \KL{P_{\vect{\Theta}\mid\vect{Z}_k = \vect{z}_k}^{(Q, \lambda_k)}}{Q} \right\rbrace.
\end{IEEEeqnarray}
The relevance of Gibbs algorithms, in particular when~$\lambda_k > 0$, stems from the observation that the probability measure~$P_{\vect{\Theta}\mid \vect{Z}_k = \vect{z}_k}^{(Q, \lambda_k)}$ in~\eqref{EqGibbsRNDFixedDataset} is the long-run distribution of a stochastic gradient descent algorithm, as highlighted in~\cite{Azizian2024What}. 
\section{Main result} 
\label{SecMainResult}
The main result of this work is presented using the following~$(\mathsf{L}_0, Q,\lambda_0)$-Gibbs probability measure 
\begin{subequations}\label{EqJanuary12at11h30in2026Sophia}
\begin{IEEEeqnarray}{rcl}
P^{\autoparent{Q, \lambda_0}}_{\vect{\Theta} \mid \vect{Z}_0 = \vect{z}_0}  \in \simplex{\set{M}},
\end{IEEEeqnarray}
with~$\lambda_0 > 0$, which represents a Gibbs algorithm trained upon the aggregated dataset~$\vect{z}_0$ in~\eqref{EqDatasetZero}. That is, for all~$\vect{\theta}\in\supp Q$, 
\begin{IEEEeqnarray}{rCl}
\label{EqJanuary14at09h53in2026Nice}
\RND{P^{\autoparent{Q,  \lambda_0}}_{\vect{\Theta} \mid \vect{Z}_0 = \vect{z}_0}}{Q}\autoparent{\vect{\theta}} &=& \frac{ \exp\autoparent{ -\frac{1}{\lambda_0} \EmpRisk{0}{\vect{z}_0}{\vect{\theta}}}} { \int \exp\autoparent{ -\frac{1}{\lambda_0}\EmpRisk{0}{\vect{z}_0}{\vect{\nu}}}\mathrm{d} Q \autoparent{\vect{\nu}}}.
\end{IEEEeqnarray}
\end{subequations}
From Lemma~\ref{LemmaJanuary1at14h24in2026}, it follows that such an algorithm is optimal in the sense that the corresponding probability measure~$P^{\autoparent{Q, \lambda_0}}_{\vect{\Theta} \mid \vect{Z}_0 = \vect{z}_0}$ in~\eqref{EqJanuary12at11h30in2026Sophia} minimizes the expected training empirical risk with respect to~$\vect{z}_0$,  over all probability measures in the following neighborhood of~$Q$,
\begin{IEEEeqnarray}{rCl}
\label{EqJanuary13at16h25in2026BusToNice}
\mathcal{G}_0& \triangleq &\left\lbrace P \in \triangle_{Q}(\set{M}) :  \KL{P}{Q} \leqslant \KL{P^{\autoparent{Q, \lambda_0}}_{\vect{\Theta} \mid \vect{Z}_0 = \vect{z}_0}}{Q}\right\rbrace. \spnum
\end{IEEEeqnarray}
\enlargethispage{-0.10in}

Consider that  the dataset~$\vect{z}_0$ in~\eqref{EqDatasetZero} is partitioned into two smaller datasets~$\vect{z}_1 \in \set{Z}^{n_1}$ and~$\vect{z}_2 \in \set{Z}^{n_2}$, \ie~$K=2$.  
The objective is to transform the~$(\mathsf{L}_0, Q,\lambda_0)$-Gibbs  probability measure~$P^{\autoparent{Q, \lambda_0}}_{\vect{\Theta} \mid \vect{Z}_0 = \vect{z}_0}$ in~\eqref{EqJanuary12at11h30in2026Sophia} into  an~$(\mathsf{L}_2, Q,\lambda_2)$-Gibbs probability measure 
\begin{subequations}
\label{EqJanuary12at11h45in2026Sophia}
\begin{IEEEeqnarray}{rcl}
P^{\autoparent{Q, \lambda_2}}_{\vect{\Theta} \mid \vect{Z}_2 = \vect{z}_2}  \in \simplex{\set{M}},
\end{IEEEeqnarray}
for some~$\lambda_2 > 0$, which represents a Gibbs algorithm exclusively trained upon dataset~$\vect{z}_2$.  That is, for all~$\vect{\theta}\in\supp Q$, 
\begin{IEEEeqnarray}{rCl}
\label{EqDec30at15h43in2025Antibes}
\RND{P^{\autoparent{Q,  \lambda_2}}_{\vect{\Theta} \mid \vect{Z}_2 = \vect{z}_2}}{Q}\autoparent{\vect{\theta}} &=& \frac{ \exp\autoparent{ -\frac{1}{\lambda_2} \EmpRisk{2}{\vect{z}_2}{\vect{\theta}}}} { \int \exp\autoparent{ -\frac{1}{\lambda_2}\EmpRisk{2}{\vect{z}_2}{\vect{\nu}}}\mathrm{d} Q \autoparent{\vect{\nu}}}.
\end{IEEEeqnarray}
\end{subequations}
From Lemma~\ref{LemmaJanuary1at14h24in2026}, it follows that the Gibbs algorithm represented by the measure~$P^{\autoparent{Q, \lambda_2}}_{\vect{\Theta} \mid \vect{Z}_2 = \vect{z}_2}$ in~\eqref{EqJanuary12at11h45in2026Sophia} minimizes the expected empirical risk with respect to~$\vect{z}_2$  over all probability measures in the following neighborhood of~$Q$,
\begin{IEEEeqnarray}{rCl}
\label{EqJanuary143at56h26in2026Nice}
\mathcal{G}_1& \triangleq &
\left\lbrace P \in \simplexabs{Q}{\set{M}} :  \KL{P}{Q} \leqslant \KL{P^{\autoparent{Q, \lambda_2}}_{\vect{\Theta} \mid \vect{Z}_2 = \vect{z}_2} }{Q}\right\rbrace.\spnum
\end{IEEEeqnarray}

In a nutshell, the objective is to unlearn the dataset~$\vect{z}_1$ from the Gibbs algorithm represented by the probability measure~$P^{\autoparent{Q, \lambda_0}}_{\vect{\Theta} \mid \vect{Z}_0 = \vect{z}_0}$ in~\eqref{EqJanuary12at11h30in2026Sophia}, while retaining the dataset~$\vect{z}_2$. This operation is often referred to as exact unlearning and is defined for this particular case as follows:
\begin{definition}[Exact unlearning]\label{DefJanuary16at19h47in2026Sophia}
Consider the algorithm represented by the $(\mathsf{L}_0, Q,\lambda_0)$-Gibbs probability measure $P^{\autoparent{Q, \lambda_0}}_{\vect{\Theta} \mid \vect{Z}_0 = \vect{z}_0}$ in \eqref{EqJanuary12at11h30in2026Sophia} and the datasets $\vect{z}_1$ and $\vect{z}_2$ to be unlearned and retained, respectively, with $\vect{z}_0$, $\vect{z}_1$ and $\vect{z}_2$ in \eqref{EqDatasetZero}. 
An algorithm $P \in \simplexabs{Q}{\set{M}}$ is said to achieve exact unlearning of $\vect{z}_1$ from  $P^{\autoparent{Q, \lambda_0}}_{\vect{\Theta} \mid \vect{Z}_0 = \vect{z}_0}$, if for all $\vect{\theta}\in\supp Q$,
\begin{IEEEeqnarray}{rCl}
\label{EqUnlearningDef}
\frac{\mathrm{d}P}{\mathrm{d}P^{\autoparent{Q,\lambda_2}}_{\vect{\Theta}\mid \vect{Z}_2=\vect{z}_2}}\autoparent{\vect{\theta}} = 1,
\end{IEEEeqnarray}
where the probability measure $P^{\autoparent{Q,\lambda_2}}_{\vect{\Theta}\mid \vect{Z}_2=\vect{z}_2}$ is defined in \eqref{EqJanuary12at11h45in2026Sophia}, for some $\lambda_2 > 0$.
\end{definition}
The equality in \eqref{EqUnlearningDef} implies that the measures $P$ and $P^{\autoparent{Q,\lambda_2}}_{\vect{\Theta}\mid \vect{Z}_2=\vect{z}_2}$ are identical.   See \cite[Theorem~3]{InriaRR9591}.

The main result of this work consists in showing that exact unlearning of $\vect{z}_1$ from  $P^{\autoparent{Q, \lambda_0}}_{\vect{\Theta} \mid \vect{Z}_0 = \vect{z}_0}$ in \eqref{EqJanuary12at11h30in2026Sophia} is achieved by 
solving the optimization problem
\begin{IEEEeqnarray}{rCl}
\label{EqJan6at15h24in2026Sophia}
\max_{P \in \simplexabs{Q}{\set{M}}}&&
\mathsf{R}_{\vect{z}_1} \autoparent{P}
+ \lambda_1 \KL{P}{P^{\autoparent{Q,\lambda_0}}_{\vect{\Theta}\mid \vect{Z}_0=\vect{z}_0}},
\end{IEEEeqnarray}
for some~$\lambda_1<0$, where the functional~$\mathsf{R}_{\vect{z}_1}$ is defined in~\eqref{EqExpEmpRiskModels}.
From Lemma~\ref{LemmaJan8at15h57in2026inSophia}, it follows that the unique solution to~\eqref{EqJan6at15h24in2026Sophia} is an~$(\mathsf{L}_1, P^{\autoparent{Q,\lambda_0}}_{\vect{\Theta}\mid \vect{Z}_0=\vect{z}_0},\lambda_1)$-Gibbs  probability measure
\begin{IEEEeqnarray}{rCl}
\label{EqJanuary14at9h22in2026HomeNice}
P^{\autoparent{P^{\autoparent{Q, \lambda_0}}_{\vect{\Theta} \mid \vect{Z}_0 = \vect{z}_0},\lambda_1}}_{\vect{\Theta} \mid \vect{Z}_1 = \vect{z}_1} & \in & \simplex{\set{M}}.
\end{IEEEeqnarray}

The following theorem, which is the main result of this work, shows that a specific choice of $\lambda_1$ and $\lambda_2$ makes the~$(\mathsf{L}_1, P^{\autoparent{Q,\lambda_0}}_{\vect{\Theta}\mid \vect{Z}_0=\vect{z}_0},\lambda_1)$-Gibbs probability  measure in~\eqref{EqJanuary14at9h22in2026HomeNice}  identical to the probability measure~$P^{\autoparent{Q, \lambda_2}}_{\vect{\Theta} \mid \vect{Z}_2 = \vect{z}_2}$ in~\eqref{EqJanuary12at11h45in2026Sophia}, which implies exact unlearning (Definition~\ref{DefJanuary16at19h47in2026Sophia}).
\begin{theorem}
\label{TheoremDec30at21h00in2025inAntibes}
Assume that
\begin{IEEEeqnarray}{rcl}
\label{EqJanuary14at10h32in2026inNice}
\lambda_1 =-\frac{n_0}{n_1} \lambda_0 & \text{ and }& \lambda_2  = \frac{n_0}{n_2} \lambda_0,
\end{IEEEeqnarray}
for some~$\lambda_0 > 0$.
Then, the~$(\mathsf{L}_2, Q,\lambda_2)$-Gibbs probability measure~$P^{\autoparent{Q, \lambda_2}}_{\vect{\Theta} \mid \vect{Z}_2 = \vect{z}_2}$ in~\eqref{EqJanuary12at11h45in2026Sophia}  and the~$(\mathsf{L}_1, P^{\autoparent{Q,\lambda_0}}_{\vect{\Theta}\mid \vect{Z}_0=\vect{z}_0},\lambda_1)$-Gibbs  probability measure~$P^{\autoparent{P^{\autoparent{Q, \lambda_0}}_{\vect{\Theta} \mid \vect{Z}_0 = \vect{z}_0},\lambda_1}}_{\vect{\Theta} \mid \vect{Z}_1 = \vect{z}_1}$ in~\eqref{EqJanuary14at9h22in2026HomeNice}  satisfy  for all~$\vect{\theta}\in\supp Q$,  
\begin{IEEEeqnarray}{rCl}
\label{EqJanuary1at16h30in2026atHomeNice}
\RND{P^{\autoparent{P^{\autoparent{Q, \lambda_0}}_{\vect{\Theta} \mid \vect{Z}_0 = \vect{z}_0},\lambda_1}}_{\vect{\Theta} \mid \vect{Z}_1 = \vect{z}_1}}{P^{\autoparent{Q,  \lambda_2}}_{\vect{\Theta} \mid \vect{Z}_2 = \vect{z}_2}}\autoparent{\vect{\theta}} &=& 1.
\end{IEEEeqnarray}
\end{theorem}
\begin{IEEEproof}
The proof is presented in Section~\ref{SecProofofTheoremDec30at21h00in2025inAntibes}.
\end{IEEEproof}

Theorem~\ref{TheoremDec30at21h00in2025inAntibes} shows that the probability measure~$P^{\autoparent{Q,\lambda_0}}_{\vect{\Theta}\mid \vect{Z}_0=\vect{z}_0}$ in~\eqref{EqJanuary12at11h30in2026Sophia} can be transformed into 
$P^{\autoparent{Q, \lambda_2}}_{\vect{\Theta} \mid \vect{Z}_2 = \vect{z}_2}$ in~\eqref{EqJanuary12at11h45in2026Sophia} via the optimization problem in~\eqref{EqJan6at15h24in2026Sophia}, achieving exact unlearning in the sense of Definition~\ref{DefJanuary16at19h47in2026Sophia}. 
An important observation is that only the dataset to be unlearned, i.e., dataset~$\vect{z}_1$, is used in the optimization problem in~\eqref{EqJan6at15h24in2026Sophia}. The dependence of such an optimization problem on the dataset to be retained, i.e.,~$\vect{z}_2$,  is via the reference measure $P^{\autoparent{Q,\lambda_0}}_{\vect{\Theta}\mid \vect{Z}_0=\vect{z}_0}$, which represents the Gibbs algorithm from which the dataset~$\vect{z}_1$ must be unlearned.  
Hence, the optimization problem in \eqref{EqJan6at15h24in2026Sophia} implies that unlearning the dataset~$\vect{z}_1$ from the Gibbs algorithm~$P^{\autoparent{Q, \lambda_0}}_{\vect{\Theta} \mid \vect{Z}_0 = \vect{z}_0}$ is achieved through training a new Gibbs algorithm exclusively upon the dataset to be unlearned, i.e., dataset~$\vect{z}_1$, while using as reference measure the Gibbs algorithm~$P^{\autoparent{Q, \lambda_0}}_{\vect{\Theta} \mid \vect{Z}_0 = \vect{z}_0}$ itself.
Interestingly, this establishes a proof of exact unlearning as, under the scaling in \eqref{EqJanuary14at10h32in2026inNice}, the~$(\mathsf{L}_1, P^{\autoparent{Q,\lambda_0}}_{\vect{\Theta}\mid \vect{Z}_0=\vect{z}_0},\lambda_1)$-Gibbs  probability measure~$P^{\autoparent{P^{\autoparent{Q, \lambda_0}}_{\vect{\Theta} \mid \vect{Z}_0 = \vect{z}_0},\lambda_1}}_{\vect{\Theta} \mid \vect{Z}_1 = \vect{z}_1}$ in~\eqref{EqJanuary14at9h22in2026HomeNice} coincides with the Gibbs measure that would have been obtained by training only on the retained dataset~$\vect{z}_2$, i.e., the measure~$P^{\autoparent{Q, \lambda_2}}_{\vect{\Theta} \mid \vect{Z}_2 = \vect{z}_2}$ in~\eqref{EqJanuary12at11h45in2026Sophia}.  

Another consequence of Theorem~\ref{TheoremDec30at21h00in2025inAntibes}, and the uniqueness of the solutions to \eqref{EqJul25at9h55in2025NiceA} and \eqref{EqJan6at15h24in2026Sophia} implied by Lemma~\ref{LemmaJan8at15h57in2026inSophia}, is that
\begin{IEEEeqnarray}{rCl}
\nonumber
& &\argmax_{P \in \simplexabs{Q}{\set{M}}} \mathsf{R}_{\vect{z}_1} \autoparent{P} -\frac{n_0}{n_1} \lambda_0\KL{P}{P^{\autoparent{Q,\lambda_0}}_{\vect{\Theta}\mid \vect{Z}_0=\vect{z}_0}} \\
\label{EqJanuary14at11h37in2026Nice}
& = &\argmin_{P \in \simplexabs{Q}{\set{M}}} 
\mathsf{R}_{\vect{z}_2} \autoparent{P} +\frac{n_0}{n_2} \lambda_0 \KL{P}{Q}.
\end{IEEEeqnarray}
From the perspective of Lemma~\ref{LemmaJanuary1at14h24in2026}, the equality in~\eqref{EqJanuary14at11h37in2026Nice} implies that 
the probability measure in $\triangle_Q\left(\set{M}\right)$ that minimizes the expected empirical risk with respect to $\vect{z}_2$ (dataset to be retained) over all measures in the set $\mathcal{G}_1$ in
\eqref{EqJanuary143at56h26in2026Nice} is identical to the measure that maximizes the expected empirical risk with respect to $\vect{z}_1$ (dataset to be unlearned) over all measures in the set 
\begin{IEEEeqnarray}{rCl}
\nonumber
& &\mathcal{G}_2 \triangleq
 \Bigg\lbrace P \in \triangle_{Q}{\left( \set{M} \right)} : \\
\label{EqJanuary13at16h26in2026BusToNice}
& &  \KL{P}{P^{\autoparent{Q, \lambda_0}}_{\vect{\Theta} \mid \vect{Z}_0 = \vect{z}_0}} \leqslant \KL{P^{\autoparent{P^{\autoparent{Q, \lambda_0}}_{\vect{\Theta} \mid \vect{Z}_0 = \vect{z}_0},-\frac{n_0}{n_1}\lambda_0}}_{\vect{\Theta} \mid \vect{Z}_1 = \vect{z}_1}}{P^{\autoparent{Q, \lambda_0}}_{\vect{\Theta} \mid \vect{Z}_0 = \vect{z}_0}}\Bigg\rbrace.\spnum
\end{IEEEeqnarray}
\enlargethispage{-0.10in}
Figure~\ref{FigGeomNestedGibbs} provides a representation of~\eqref{EqJanuary14at11h37in2026Nice} under the assumption in~\eqref{EqJanuary14at10h32in2026inNice}: it depicts how the common optimizer in~\eqref{EqJanuary14at11h37in2026Nice} can be equivalently characterized either as an empirical-risk minimizer over $\mathcal{G}_1$ in~\eqref{EqJanuary143at56h26in2026Nice} (relative to $\vect{z}_2$) or as an empirical-risk maximizer over $\mathcal{G}_2$ in~\eqref{EqJanuary13at16h26in2026BusToNice} (relative to $\vect{z}_1$), while also highlighting the relation with $\mathcal{G}_0$ in~\eqref{EqJanuary13at16h25in2026BusToNice}. In particular, the figure shows the sets~$\mathcal{G}_0$, $\mathcal{G}_1$, and $\mathcal{G}_2$ together with the probability measures~$P^{\autoparent{Q, \lambda_0}}_{\vect{\Theta} \mid \vect{Z}_0 = \vect{z}_0}$ in~\eqref{EqJanuary12at11h30in2026Sophia}, $P^{\autoparent{Q, \lambda_2}}_{\vect{\Theta} \mid \vect{Z}_2 = \vect{z}_2}$ in~\eqref{EqJanuary12at11h45in2026Sophia}, and  $P^{\autoparent{P^{\autoparent{Q, \lambda_0}}_{\vect{\Theta} \mid \vect{Z}_0 = \vect{z}_0},\lambda_1}}_{\vect{\Theta} \mid \vect{Z}_1 = \vect{z}_1}$ in~\eqref{EqJanuary14at9h22in2026HomeNice}.
In such a figure, the following notation is used:
\begin{IEEEeqnarray}{rCl}\label{EqJan16at7h01in2026inSophiaA}
\alpha_0 &\triangleq& \KL{P^{\autoparent{Q, \lambda_0}}_{\vect{\Theta} \mid \vect{Z}_0 = \vect{z}_0}}{Q};\\
\label{EqJan16at7h01in2026inSophiaB}
\alpha_1 &\triangleq& \KL{P^{\autoparent{Q, \frac{n_0}{n_2}\lambda_{0}}}_{\vect{\Theta} \mid \vect{Z}_2 = \vect{z}_2}}{Q}; \mbox{ and }\\
\label{EqJan16at7h01in2026inSophiaC}
\alpha_2 &\triangleq& \KL{P^{\autoparent{P^{\autoparent{Q, \lambda_0}}_{\vect{\Theta} \mid \vect{Z}_0 = \vect{z}_0},-\frac{n_0}{n_1}\lambda_{0}}}_{\vect{\Theta} \mid \vect{Z}_1 = \vect{z}_1}}{P^{\autoparent{Q, \lambda_0}}_{\vect{\Theta} \mid \vect{Z}_0 = \vect{z}_0}},
\end{IEEEeqnarray}
under the assumption that $\alpha_0 < \alpha_1$.

Finally, note that the parameter $\lambda_0 > 0$ remains a free parameter that can be further optimized, e.g., to minimize the generalization error of the resulting Gibbs algorithm $P^{\autoparent{P^{\autoparent{Q, \lambda_0}}_{\vect{\Theta} \mid \vect{Z}_0 = \vect{z}_0},-\frac{n_0}{n_1}\lambda_0}}_{\vect{\Theta} \mid \vect{Z}_1 = \vect{z}_1}$. See for instance \cite{perlaza2024generalization,aminian2024information,zouJSAIT2024,aminian2021exact,xu2017information} and references therein for this purpose.

\begin{figure}[t]
\centering
 \begin{tikzpicture}[scale=1.05, line cap=round, line join=round]
  \coordinate (Q0) at (0,0);
  \coordinate (P0) at (1.0,0.10);
  \coordinate (P2) at (1.5,1.5);

  \draw[thick, black]
    let \p1=(Q0), \p2=(P0), \n1={veclen(\x2-\x1,\y2-\y1)} in
  (Q0) ellipse ({\n1} and {\n1});

  \draw[dashed, thick, red!70]
    let \p1=(Q0), \p2=(P2), \n1={veclen(\x2-\x1,\y2-\y1)} in
  (Q0) ellipse ({\n1} and {\n1});

  \draw[dotted, thick, blue!70]
    let \p1=(P0), \p2=(P2), \n1={veclen(\x2-\x1,\y2-\y1)} in
  (P0) ellipse ({\n1} and {\n1});

  \draw[thick, black]   (Q0) -- node[midway, fill=white, inner sep=1pt] {$\alpha_0$} (P0);
  \draw[thick, red!70]  (Q0) -- node[midway, fill=white, inner sep=1pt, text=red!70] {$\alpha_1$} (P2);
  \draw[thick, blue!70] (P0) -- node[midway, fill=white, inner sep=1pt, text=blue!70] {$\alpha_2$} (P2);

  \node[cross out, draw=black, thick, inner sep=1.6pt] at (Q0) {};
  \node[cross out, draw=blue!70, thick, inner sep=1.6pt] at (P0) {};
  \node[cross out, draw=red!70, thick, inner sep=1.8pt] at (P2) {};
  \node[cross out, draw=green!50!black, thick, inner sep=0.9pt] at (P2) {};

  \node[] at (-0.7,0.3) {$\mathcal{G}_0$};
  \node[below right=0pt] at (Q0) {$Q$};

  \node[text=blue!70] at (0.5,1.2) {$\mathcal{G}_2$};
  \node[below right=-5pt, text=blue!70] at (P0)
    {$P^{(Q,\lambda_0)}_{\vect{\Theta}\mid \vect{Z}_0=\vect{z}_0}$};

  \node[text=red!70] at (-1,1) {$\mathcal{G}_1$};
  \node[above right, text=red!70] at (P2)
    {$P^{\!\left(Q,\frac{n_0}{n_2}\lambda_0\right)}_{\vect{\Theta}\mid \vect{Z}_2=\vect{z}_2}
    \textcolor{black}{ = }
    \textcolor{green!50!black}{P^{\!\left(P^{(Q,\lambda_0)}_{\vect{\Theta}\mid \vect{Z}_0=\vect{z}_0},-\frac{n_0}{n_1}\lambda_0\right)}_{\vect{\Theta}\mid \vect{Z}_1=\vect{z}_1}}$};
\end{tikzpicture}
\caption{Representation, under the assumption in~\eqref{EqJanuary14at10h32in2026inNice}, of the set $\mathcal{G}_0$ in~\eqref{EqJanuary13at16h25in2026BusToNice} as a solid black circle; the set $\mathcal{G}_1$ in~\eqref{EqJanuary143at56h26in2026Nice} as a dashed red circle; and the set $\mathcal{G}_2$ in~\eqref{EqJanuary13at16h26in2026BusToNice} as a dotted blue circle. The involved probability measures are defined in~\eqref{EqJanuary12at11h30in2026Sophia}, \eqref{EqJanuary12at11h45in2026Sophia}, and~\eqref{EqJanuary14at9h22in2026HomeNice}. The definitions of $\alpha_0,\alpha_1,$ and $\alpha_2$ are in~\eqref{EqJan16at7h01in2026inSophiaA}, \eqref{EqJan16at7h01in2026inSophiaB} and \eqref{EqJan16at7h01in2026inSophiaC}, respectively.}
\label{FigGeomNestedGibbs}
\vspace{-3ex}
\end{figure}
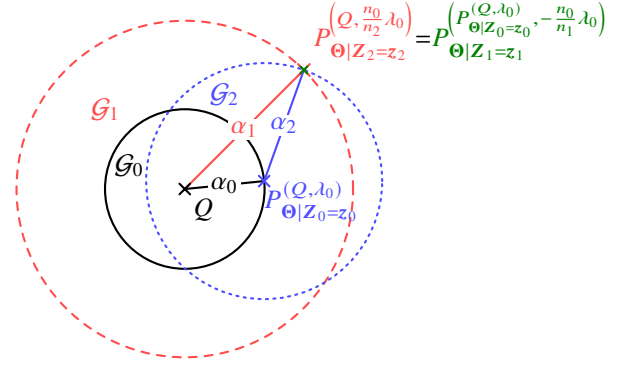

\subsection{Proof of Theorem~\ref{TheoremDec30at21h00in2025inAntibes}}\label{SecProofofTheoremDec30at21h00in2025inAntibes}

The proof of Theorem~\ref{TheoremDec30at21h00in2025inAntibes} relies on the following result.
\enlargethispage{-0.10in}
\begin{theorem}
\label{TheoremJan6at20h59in2026Antibes}
The~$(\mathsf{L}_1, P^{\autoparent{Q,\lambda_0}}_{\vect{\Theta}\mid \vect{Z}_0=\vect{z}_0},\lambda_1)$-Gibbs  probability measure~$P^{\autoparent{P^{\autoparent{Q, \lambda_0}}_{\vect{\Theta} \mid \vect{Z}_0 = \vect{z}_0},\lambda_1}}_{\vect{\Theta} \mid \vect{Z}_1 = \vect{z}_1}$ in~\eqref{EqJanuary14at9h22in2026HomeNice}, with $\lambda_0 > 0$ and $\lambda_1 \in \reals\setminus\lbrace 0 \rbrace$, satisfies for all~$\vect{\theta}\in\supp Q$,  
\begin{IEEEeqnarray}{rCl}
\label{EqJan2at19h24in2026inLaCadiere}
\nonumber
&&\RND{P^{\autoparent{P^{\autoparent{Q, \lambda_0}}_{\vect{\Theta} \mid \vect{Z}_0 = \vect{z}_0}, \lambda_1}}_{\vect{\Theta} \mid \vect{Z}_1 = \vect{z}_1}}{Q}(\vect{\theta}) \\
&=&  \frac{ \exp\autoparent{ \frac{-1}{\lambda_0} \autoparent{\autoparent{\frac{\lambda_0}{\lambda_1} + \frac{n_1}{ n_0}}\EmpRisk{1}{\vect{z}_1}{\vect{\theta}} + \frac{n_2}{n_0}\EmpRisk{2}{\vect{z}_2}{\vect{\theta}}}}} { \int \exp\autoparent{ \frac{-1}{\lambda_0} \autoparent{\autoparent{\frac{\lambda_0}{\lambda_1} + \frac{n_1}{n_0}}\EmpRisk{1}{\vect{z}_1}{\vect{\nu}} + \frac{n_2}{n_0}\EmpRisk{2}{\vect{z}_2}{\vect{\nu}}}}\mathrm{d} Q \autoparent{\vect{\nu}}}, \squeezeequ \spnum
\end{IEEEeqnarray}
where the function $\mathsf{L}_{2}$ is defined in \eqref{EqEmpiricalRiskk}.
\end{theorem}
\begin{IEEEproof}
The proof is presented in \cite{InriaRR9610}.
\end{IEEEproof}

Theorem~\ref{TheoremJan6at20h59in2026Antibes} makes explicit the role of~$\lambda_1$ as a reweighting parameter of the contribution of the dataset $\vect{z}_1$.
The exact unlearning identity in
Theorem~\ref{TheoremDec30at21h00in2025inAntibes} follows by choosing~$\lambda_1$ to completely eliminate the contribution of dataset $\vect{z}_1$. That is, $\lambda_1$ is chosen such that
\vspace{-1ex}
\begin{IEEEeqnarray}{rCl}
\label{EqJanuary15at13h53in2026Sophia}
\frac{\lambda_0}{\lambda_1} + \frac{n_1}{n_0} & = & 0.
\end{IEEEeqnarray}
More specifically, using~\eqref{EqDec30at15h43in2025Antibes} and~\eqref{EqJanuary15at13h53in2026Sophia} in~\eqref{EqJan2at19h24in2026inLaCadiere} yields, 
\begin{IEEEeqnarray}{rCl}
\label{EqJanuary15at13h55in2026Sophia}
 \RND{P^{\autoparent{P^{\autoparent{Q, \lambda_0}}_{\vect{\Theta} \mid \vect{Z}_0 = \vect{z}_0}, -\frac{n_0}{n_1}\lambda_0}}_{\vect{\Theta} \mid \vect{Z}_1 = \vect{z}_1}}{Q}(\vect{\theta})  & = & \RND{P^{\autoparent{Q, \frac{n_0}{n_2}\lambda_0}}_{\vect{\Theta} \mid \vect{Z}_2 = \vect{z}_2}}{Q}(\vect{\theta}).
\end{IEEEeqnarray}
From~\cite[Lemma~3]{perlaza2024empirical} and \eqref{EqJanuary15at13h55in2026Sophia}, it follows that 
$Q \ll P^{\autoparent{Q, \lambda_0}}_{\vect{\Theta} \mid \vect{Z}_0 = \vect{z}_0} \ll P^{\autoparent{P^{\autoparent{Q, \lambda_0}}_{\vect{\Theta} \mid \vect{Z}_0 = \vect{z}_0}, -\frac{n_0}{n_1}\lambda_0}}_{\vect{\Theta} \mid \vect{Z}_1 = \vect{z}_1}  \ll P^{\autoparent{Q, \frac{n_0}{n_2}\lambda_0}}_{\vect{\Theta} \mid \vect{Z}_2 = \vect{z}_2} \ll Q.
$
%
Hence, from~\eqref{EqJanuary15at13h55in2026Sophia}, it follows that
\begin{IEEEeqnarray}{rCl}
1
\label{EqJanuary15at14h29in2026SophiaA}
& = & \RND{P^{\autoparent{P^{\autoparent{Q, \lambda_0}}_{\vect{\Theta} \mid \vect{Z}_0 = \vect{z}_0}, -\frac{n_0}{n_1}\lambda_0}}_{\vect{\Theta} \mid \vect{Z}_1 = \vect{z}_1}}{Q}(\vect{\theta})  \RND{Q}{P^{\autoparent{Q, \frac{n_0}{n_2}\lambda_0}}_{\vect{\Theta} \mid \vect{Z}_2 = \vect{z}_2}}(\vect{\theta})\\
\label{EqJanuary15at14h26in2026Sophia}
& = & \RND{P^{\autoparent{P^{\autoparent{Q, \lambda_0}}_{\vect{\Theta} \mid \vect{Z}_0 = \vect{z}_0}, -\frac{n_0}{n_1}\lambda_0}}_{\vect{\Theta} \mid \vect{Z}_1 = \vect{z}_1}}{P^{\autoparent{Q, \frac{n_0}{n_2}\lambda_0}}_{\vect{\Theta} \mid \vect{Z}_2 = \vect{z}_2}}(\vect{\theta}),
 \end{IEEEeqnarray}
where the equality in~\eqref{EqJanuary15at14h29in2026SophiaA} follows from~\cite[Theorem~5]{InriaRR9591}; and 
the equality in~\eqref{EqJanuary15at14h26in2026Sophia} follows from~\cite[Theorem~4]{InriaRR9591}.
This completes the proof of Theorem~\ref{TheoremDec30at21h00in2025inAntibes}.

\subsection{Reweighting Data Points in ERM-RER}

Beyond the specific choices of $\lambda_1$ and $\lambda_2$ in \eqref{EqJanuary14at10h32in2026inNice}, in general if~$\lambda_0 > 0$ and~$\lambda_1 \in  \reals\setminus\lbrace 0 \rbrace$,  Theorem~\ref{TheoremJan6at20h59in2026Antibes} unveils the following observation:   
\begin{IEEEeqnarray}{rCl}
\nonumber
&  &\argmin_{P \in \simplexabs{Q}{\set{M}}} \autoparent{\frac{\lambda_0}{\lambda_1} + \frac{n_1}{n_0}} \mathsf{R}_{\vect{z}_1} (P) +   \frac{n_2}{n_0}\mathsf{R}_{\vect{z}_2}(P) + \lambda_0 \KL{P}{Q} \\
\label{EqJanuary14at15h42in2026Nice}
& = &\argmin_{P \in \simplexabs{Q}{\set{M}}} \frac{\lambda_0}{\lambda_1} \mathsf{R}_{\vect{z}_1} (P) + \mathsf{R}_{\vect{z}_0}(P) + \lambda_0 \KL{P}{Q} \\
\label{EqJan3at15h24in2026LaCadiere}
&  = &\left\lbrace
\begin{array}{rc}
\displaystyle\argmin_{P \in \triangle_{Q} \autoparent{\set{M}}}  \mathsf{R}_{\vect{z}_1} \left( P \right)
+ \lambda_1 \KL{P}{P^{\autoparent{Q, \lambda_0}}_{\vect{\Theta} \mid \vect{Z}_0 = \vect{z}_0}} &\text{if } \lambda_1 > 0;\\
\displaystyle\argmax_{P \in \triangle_{Q} \autoparent{\set{M}}}  \mathsf{R}_{\vect{z}_1} \left( P \right)
+ \lambda_1 \KL{P}{P^{\autoparent{Q, \lambda_0}}_{\vect{\Theta} \mid \vect{Z}_0 = \vect{z}_0}} &\text{if } \lambda_1 < 0
\end{array}
\right. \\
\label{EqJanuary14at15h46in2026Nice}
& = & \left\lbrace P^{\autoparent{P^{\autoparent{Q, \lambda_0}}_{\vect{\Theta} \mid \vect{Z}_0 = \vect{z}_0},\lambda_1}}_{\vect{\Theta} \mid \vect{Z}_1 = \vect{z}_1} \right\rbrace,
\end{IEEEeqnarray} 
where the equality in \eqref{EqJanuary14at15h42in2026Nice} follows from the fact that for all $P  \in \triangle_{Q}\left( \set{M} \right)$, 
\begin{IEEEeqnarray}{rCl}
\mathsf{R}_{\vect{z}_0} (P)  & = & \frac{n_1}{n_0} \mathsf{R}_{\vect{z}_1} (P) +   \frac{n_2}{n_0}\mathsf{R}_{\vect{z}_2}(P) ; 
\vspace{1ex}
\end{IEEEeqnarray}
the equality in~\eqref{EqJan3at15h24in2026LaCadiere} follows from Theorem~\ref{TheoremJan6at20h59in2026Antibes} and Lemma~\ref{LemmaJan8at15h57in2026inSophia}; and 
the equality in~\eqref{EqJanuary14at15h46in2026Nice}  follows from  Lemma~\ref{LemmaJan8at15h57in2026inSophia}.
\enlargethispage{-\baselineskip}
The equality in \eqref{EqJanuary14at15h42in2026Nice} highlights that both $\lambda_0$ and $\lambda_1$ reweight the dataset~$\vect{z}_1$ in the optimization problem whose solution is the measure~$P^{\autoparent{P^{\autoparent{Q, \lambda_0}}_{\vect{\Theta} \mid \vect{Z}_0 = \vect{z}_0},\lambda_1}}_{\vect{\Theta} \mid \vect{Z}_1 = \vect{z}_1}$ in~\eqref{EqJanuary14at9h22in2026HomeNice}, relative to the optimization problem whose solution is~$P^{\autoparent{Q, \lambda_0}}_{\vect{\Theta} \mid \vect{Z}_0 = \vect{z}_0}$ in~\eqref{EqJanuary12at11h30in2026Sophia}. 
More specifically, from Lemma~\ref{LemmaJan8at15h57in2026inSophia}, it holds that
\begin{IEEEeqnarray}{rCl}
\nonumber
\autobrace{P^{\autoparent{Q, \lambda_0}}_{\vect{\Theta} \mid \vect{Z}_0 = \vect{z}_0}} & = &\argmin_{P \in \simplexabs{Q}{\set{M}}} \Bigg( \frac{n_1}{n_0}  \mathsf{R}_{\vect{z}_1} \autoparent{P}+ \frac{n_2}{n_0} \mathsf{R}_{\vect{z}_2} \autoparent{P}\Bigg)\\
\label{EqJan9at8h28in2026Sophia}
&&+ \lambda_0 \KL{P}{Q},
\end{IEEEeqnarray}
and from~\eqref{EqJanuary14at15h46in2026Nice}, it holds that
\begin{IEEEeqnarray}{rcl}
\nonumber
\autobrace{P^{\autoparent{P^{\autoparent{Q, \lambda_0}}_{\vect{\Theta} \mid \vect{Z}_0 = \vect{z}_0},\lambda_1}}_{\vect{\Theta} \mid \vect{Z}_1 = \vect{z}_1}} & =  &\argmin_{P \in \simplexabs{Q}{\set{M}}}
 \autoparent{\frac{\lambda_0}{\lambda_1} + \frac{n_1}{n_0}} \mathsf{R}_{\vect{z}_1} (P) +   \frac{n_2}{n_0}\mathsf{R}_{\vect{z}_2}(P) \\
\label{EqJanuary15at10h53in2026Sophia}
 & & + \lambda_0 \KL{P}{Q}.
\end{IEEEeqnarray}

The contribution of the dataset $\vect{z}_1$ in~\eqref{EqJan9at8h28in2026Sophia} is proportional to~$\frac{n_1}{n_0}$, while in~\eqref{EqJanuary15at10h53in2026Sophia}, it is proportional to $\frac{\lambda_0}{\lambda_1} + \frac{n_1}{n_0}$. 
Interestingly, when $\lambda_1 > 0$, the contribution of the dataset $\vect{z}_1$ is up-weighted in~\eqref{EqJanuary15at10h53in2026Sophia}. That is, 
$\frac{\lambda_0}{\lambda_1} + \frac{n_1}{n_0} > \frac{n_1}{n_0}$.
The converse holds when $\lambda_1 < 0$.
It is important to highlight that the contribution of the dataset $\vect{z}_2$ remains invariant in both \eqref{EqJan9at8h28in2026Sophia} and~\eqref{EqJanuary15at10h53in2026Sophia}.

\section{Conclusions and Final Remarks}
\label{SecConclusion}

In this work, an exact unlearning method for Gibbs algorithms has been proposed in  Theorem~\ref{TheoremDec30at21h00in2025inAntibes}. The method consists in implementing a new  Gibbs algorithm via the maximization of the empirical risk with respect to the dataset to be unlearned using as reference measure the algorithm from which data must be unlearned. 
Under the corresponding scaling, this construction guarantees exact unlearning in the sense that the new Gibbs algorithm coincides with the Gibbs algorithm that would have been obtained by retraining from scratch exclusively on the dataset to be retained. 
An important feature is that the corresponding optimization problem depends exclusively on the dataset to be unlearned, while the dependence on the dataset to be retained is captured solely through the reference measure.
The proof of Theorem~\ref{TheoremDec30at21h00in2025inAntibes} relies on Theorem~\ref{TheoremJan6at20h59in2026Antibes}, which can be placed in a more general scenario than the problem of unlearning.  
In particular, it can be placed in the scenario of reweighting data points in ERM-RER~\cite{perlaza2024empirical}. From this perspective, unlearning appears as a special case in which zero weight is assigned to the contribution of the data points to be unlearned. 
More broadly, this reweighting offers a principled mechanism to modulate the contribution of particular data points, which can be leveraged to improve the generalization error by attenuating the effect of atypical or overly influential data points. Conversely, the same mechanism can be used adversarially: by up-weighting outliers, an attacker may deliberately increase the generalization error of Gibbs algorithms, yielding a model-poisoning strategy within the ERM-RER learning framework.
The strategic choice of the parameters in Theorem~\ref{TheoremJan6at20h59in2026Antibes}, beyond unlearning, is left out of the scope of this paper. 
Nonetheless, the theorem provides a unifying theoretical foundation for exact unlearning and data points reweighting in ERM-RER. This opens the door to developing both robustness-oriented reweighting and its adversarial counterpart in statistical machine learning.
\vspace{3ex}

\IEEEtriggeratref{15}
\bibliographystyle{\Latexfilepath/Bibliography/IEEEtranlink}
\bibliography{\Latexfilepath/Bibliography/Merged_BERMUDEZ_ref.bib}
 
\appendices
\end{document}